\documentclass[conference]{IEEEtran}

\usepackage{url,graphicx,subfigure,tabularx,array,amsmath,amssymb,amsthm,textcomp,booktabs}
\usepackage{natbib}
\usepackage{algorithm}
\usepackage{algorithmic}

\usepackage[hidelinks]{hyperref}

\newcommand{\PP}{\mathbb{P}}

\newtheorem{thm}{Theorem}
\newtheorem{lemma}{Lemma}

\begin{document}


\author{
\IEEEauthorblockN{George D. Monta\~nez}
\IEEEauthorblockA{Machine Learning Department\\
Carnegie Mellon University\\
Pittsburgh, PA USA\\
\texttt{gmontane@cs.cmu.edu}}
\and
\IEEEauthorblockN{Cosma Rohilla Shalizi}
\IEEEauthorblockA{Statistics Department\\
Carnegie Mellon University\\
Pittsburgh, PA USA\\
\texttt{cshalizi@cmu.edu}}
}

\title{The LICORS Cabinet: Nonparametric Light Cone Methods for Spatio-Temporal Modeling}

\maketitle

\begin{abstract}
  Spatio-temporal data is intrinsically high dimensional, so unsupervised
  modeling is only feasible if we can exploit structure in the process.  When
  the dynamics are local in both space and time, this structure can be exploited by
  splitting the global field into many lower-dimensional ``light cones''.  We
  review light cone decompositions for predictive state reconstruction,
  introducing three simple light cone algorithms. These methods allow for
  tractable inference of spatio-temporal data, such as full-frame video.  The
  algorithms make few assumptions on the underlying process yet
  have good predictive performance and can provide distributions over
  spatio-temporal data, enabling sophisticated probabilistic inference.
\end{abstract}

\section{Introduction}

Modeling spatio-temporal data, such as high resolution video, is hard. The sheer dimensionality of the data often makes global
inference methods difficult.  Similarly curses of dimensionality for textual
and time-series data have been met with great success by HMMs
\citep{rabiner1989tutorial}, using localized models for prediction and tractable
inference on sequences.  Inspired by this example, we look to localized models
for modeling of spatio-temporal data, like video and fMRI data.  \emph{Light
  cone} methods, such as mixed LICORS \citep{goerg2013mixed}, successfully
reduce the global inference task to iterating a tractable, localized one. These
methods can be used for both regression (point predictions of
$\mathbb{R}^d$-valued outputs from input variables) and computing probability
densities.  The latter property allows one to tractably compute distributions
over spaces of events, e.g., over the space of all possible videos, $V^*$, just
as HMMs induce probability distributions over the set $\Sigma^*$ of all
possible sequences (Figure~\ref{fig:video}). This ability could make light cone
decompositions as general and useful for modeling spatio-temporal data as HMMs
are for textual and time-series data.

\begin{figure}[htbp]
\centering
\includegraphics[width=0.42\linewidth]{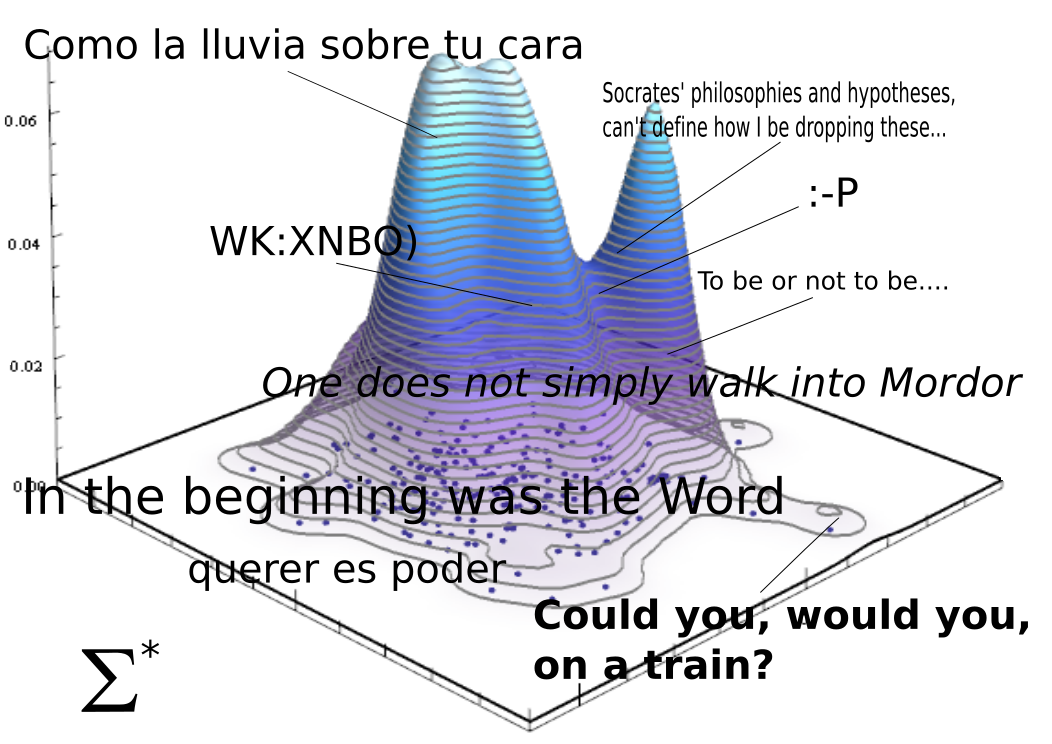}
\hspace{1em}
\includegraphics[width=0.42\linewidth]{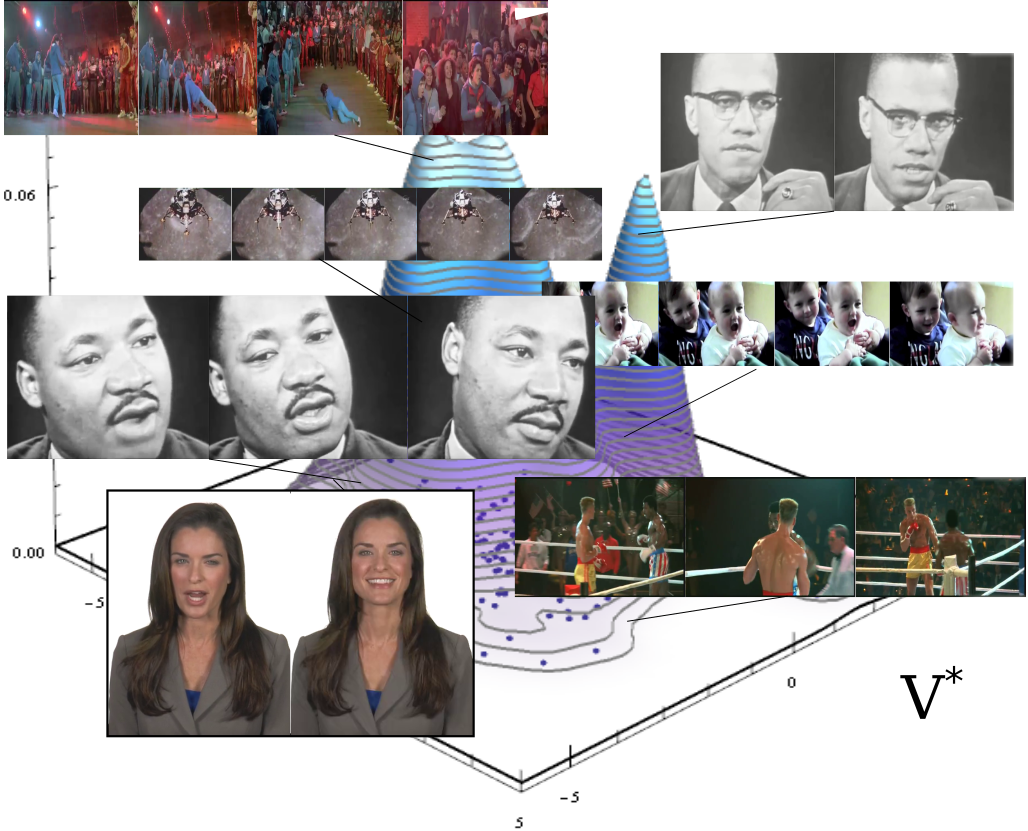}
\caption{Probability densities over the space of all strings, $\Sigma^*$, and the space of all videos, $V^*$.}
\label{fig:video}
\end{figure}

The goals of this manuscript are thus: (1) Showing how light cone decompositions help make spatio-temporal modeling tasks tractable; (2) Introducing three easy-to-implement light cone algorithms, allowing others to begin experimenting with light cone methods; (3) Assessing the predictive accuracy of light cones methods on two video prediction tasks; and (4) Providing a finite sample guarantee on the error of predictive state light cone methods. We begin with some preliminaries.

\section{Notation and Preliminaries}

Given a random field $X(\mathbf{r}, t)$, observed for each point $\mathbf{r}$
on a regular spatial lattice $\mathbf{S}$ at discrete time instants $t =
1,\ldots,T$, we seek to approximate a joint likelihood over the observations of
the spatio-temporal process, and to accurately forecast the future of the
process. Since causal influences in physical systems only propagate at finite
speed (denoted $c$), we follow \citet{Parlitz-Merkwirth-local-states} and adopt
the concept of \emph{light cones}, which are defined as the set of events that
could influence $(\mathbf{r}, t)$. Formally, a past light cone (PLC) is the set
of all past variables\footnote{Strictly, we should distinguish the light
  cone proper, which is a region of space-time, from the configuration of the
  random field over this region.  We elide the distinction for brevity.} that
could have affected $X(\mathbf{r}, t)$:
\[
    \ell^-(\mathbf{r}, t) := \{X(\mathbf{u}, s) \mid s \leq t, ||\mathbf{r} -\mathbf{u}||_2 \leq c \cdot (t-s) \}.
\]
Similarly, a future light cone (FLC) is the set of all future events that could be affected by $(\mathbf{r}, t)$. As a practical matter, not all past (or future) events are equally informative, since more recent events tend to exert greater causal influence. Thus, in practice, we can approximate the true past light cone with a much smaller subset light cone, improving tractability without incurring much predictive error.

Furthermore, we adopt the conditional independence assumption for light cones given in \citet{goerg2013mixed}, which allows for the factorization of the joint likelihood into the product of conditional likelihoods. Indexing each $X(\mathbf{r}, t)$ by a single integer $i = 1,\ldots,N$ for simplicity of notation, the joint pdf of $X_1, \ldots, X_N$ factorizes as 
\[
    P(X_1, \ldots, X_N) \propto \prod_{i=1}^{N}P(X_i | \ell^-_i), 
\]
where the proportionality accounts for incompletely observed light cones along the edge of the field.

Given this factorization, it becomes natural to seek equivalence classes of
light cones, namely, i.e., to cluster light cones into sets based on the
similarity of their conditional distributions.  Such equivalence classes of
past light cones are \emph{predictive states} \citep{Knight-predictive-view,
  goerg2012licors}, and our immediate goals become twofold: first, to discover
these latent predictive states (i.e., learn a mapping $\epsilon$ from PLCs to
predictive states), and second, to estimate the conditional distribution over
$X$ given its predictive state. \citet{goerg2012licors} introduced
\emph{LICORS} as a nonparametric method of predictive state reconstruction,
followed by \emph{mixed LICORS} \citeyearpar{goerg2013mixed} as a mixture model
extension of LICORS, where each future light cone is forecast using a mixture
of extremal predictive states. Mixed LICORS has predictive advantages over the
original LICORS, but requires finding an $N \times K$ matrix of weights (where
$N$ is the number of light cones and $K$ the number of predictive states) using
a form of EM, where each weight is determined using a kernel density estimate
on all points.  Each EM iteration takes $O(N^2 K)$ steps, slowing mixed LICORS
considerably for large $N$.  Almost equally daunting, the original algorithms
are quite complex, difficult to implement and debug, inhibiting their adoption.

\section{Contributions}

We review the use of light cones for localized spatio-temporal prediction. We
introduce two simplified nonparametric methods for the predictive state
reconstruction task and a simple regression light cone method for fast and
accurate forecasting. The first predictive state method, \textbf{Moonshine}, is
a simple meta-algorithm consisting of basic clustering steps combined with
dimensionality reduction and nonparametric density estimation. Moonshine is
instance-based and requires no iterative likelihood maximization, yet retains
many of the qualities of the more complex mixed LICORS method.  The second
predictive state algorithm, \textbf{One Hundred Proof (OHP)}, simplifies the
Moonshine approach further and consists of clustering in the space of future
light cones, using the clusters to obtain state-specific nonparametric density
estimates over the space of PLCs and FLCs. These simple algorithms are much
easier to implement than the LICORS algorithms, being simplified approximations of 
the mixed LICORS system, yet retain many of their forecasting and modeling strengths.

We further conduct two sets of empirical experiments showing the predictive
power of light cone methods for predicting video-like data, and report
results. Lastly, we give a large sample theoretical guarantee for light cone
predictive state systems.

The remainder is structured as follows. \S \ref{sec:methods} describes the Moonshine, One Hundred Proof and light cone
linear regression algorithms.  \S \ref{sec:results} describes the experimental
setup for two real-world spatio-temporal prediction tasks, and gives the
results of the algorithms and baselines. \S \ref{sec:theory} gives an upper bound on the estimation error of our methods. \S \ref{sec:RELATED-FUTURE} reviews related and future work, while \S
\ref{sec:CONCLUSION} summarizes our findings.

\section{Methods}
\label{sec:methods}

Our simple predictive state reconstruction methods build upon the principles
introduced in \citet{goerg2013mixed} for mixed LICORS.  Both new methods
reconstruct a set of predictive states and a soft mapping $\epsilon$ from past
light cones to states, through use of nonparametric density estimation over the
space of light cones.  That is, for all past light cones $\ell^-$ the methods
compute
\[
    \epsilon(\ell^-) = [w_1, w_2, \ldots, w_k]^\top,
\]
where $w_j$ is the normalized weight of state $S_j$ for light cone $\ell^-$. Unlike mixed LICORS, the new methods avoid having to explicitly construct an $N \times K$ matrix, yet retain the benefits of soft membership mixture modeling.

After describing the reconstruction algorithms, we discuss how one can
determine the conditional probability density of an observation given its past
light cone, and how to use this conditional density in forecasts.  We then
describe an additional pure regression light cone method, useful for fast and
accurate forecasting without state reconstruction. Appendix 2 describes
parameter settings and practical implementation issues that arise when using
the algorithms.

\subsection{Moonshine}

\begin{figure}[htbp]
\centering
\includegraphics[width=1.\linewidth]{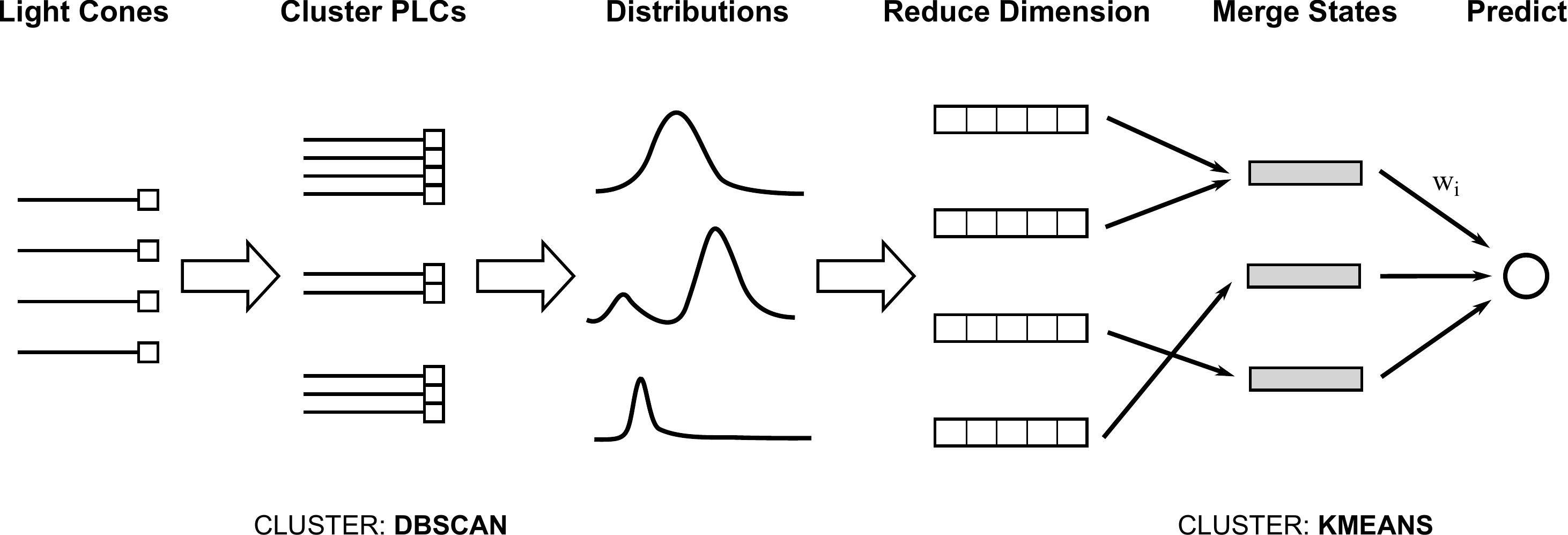}
\caption{Component stages of the Moonshine algorithm.}
\label{fig:main}
\end{figure}

\begin{algorithm}
\caption{Moonshine}
\label{alg:Moonshine}
\small
\begin{algorithmic}[1]
\STATE Decompose spatio-temporal process into light cone (PLC, FLC) observation tuples.
\STATE Cluster PLCs using density based clustering.
\STATE Compute cluster-conditioned density estimates for $2K+1$ random points.
\IF{number of clusters $>$ maximum number}
    \STATE Merge clusters in the space of reduced dimension.
\ENDIF
\STATE Map original light cones to final clusters.
\end{algorithmic}
\end{algorithm}

Moonshine begins by decomposing the random field into its component light
cones, shown at far left in Figure~\ref{fig:main}. The algorithm then proceeds
through two successive stages of clustering, separated by a dimension-reduction
step. The main steps of Moonshine are given in Algorithm~\ref{alg:Moonshine}.

The output of the procedure is a set of predictive states, each of which
consists of a set of PLCs and FLCs. The predictive states are used to create a
pair of nonparametric density estimates, one over PLCs and one over FLCs, which
jointly identify each state.

\textbf{Initial Clustering:} For the first clustering step, Moonshine uses a
density-based clustering approach \citep{DBSCAN} to cluster the light cones in
the space of PLCs, which assumes that similar PLCs have similar predictive
consequences. Such clustering methods need a specified local-neighborhood size,
so we begin with small neighborhoods, progressively increase until 90\% of all
points are clustered, and assign the remaining points to the nearest cluster
center (effectively hybridizing density-based clustering with $k$-means). This
allows for good coverage while avoiding formation of a single, all-encompassing
cluster.  (Alternative clustering algorithms, e.g., \cite{zahn1971graph,982897,zhao2015sof}, would also work.)

\textbf{Density Estimation and Dimensionality Reduction:} The FLCs associated with each
cluster (mapped through their respective light cones) are used to form kernel density
estimates over the space of FLCs. In other words, each cluster consists of some set of
associated FLCs and these FLCs are then used to estimate densities over the FLC space.
We estimate the densities of $2K + 1$ randomly selected points, where $K$ is a
parameter that affects the degree of dimensionality reduction. The
log-probability ratio is taken between the first point and the remaining
$2K$ points. This vector of log probability ratios forms the ``signature'' of
the cluster, following the construction of a canonical sufficient statistic for
exponential family distributions \citep[p.\ 123]{Kulhavy-recursive}.

\textbf{Merging Clusters:} If the number of clusters is greater than the
maximum number of predictive states specified for the model, we cluster again
to reduce the number.  We cluster the low-dimensional signature vectors with
$k$-means++ \citep{arthur2007k}, to form the final predictive states.  The
original light cones are then assigned to the resulting states, so each
predictive state has a unique set of PLCs and FLCs with which to form
nonparametric density estimates over both the PLC and FLC spaces.

\subsection{One Hundred Proof (OHP)}

\begin{figure}[htbp]
\centering
\includegraphics[width=0.85\linewidth]{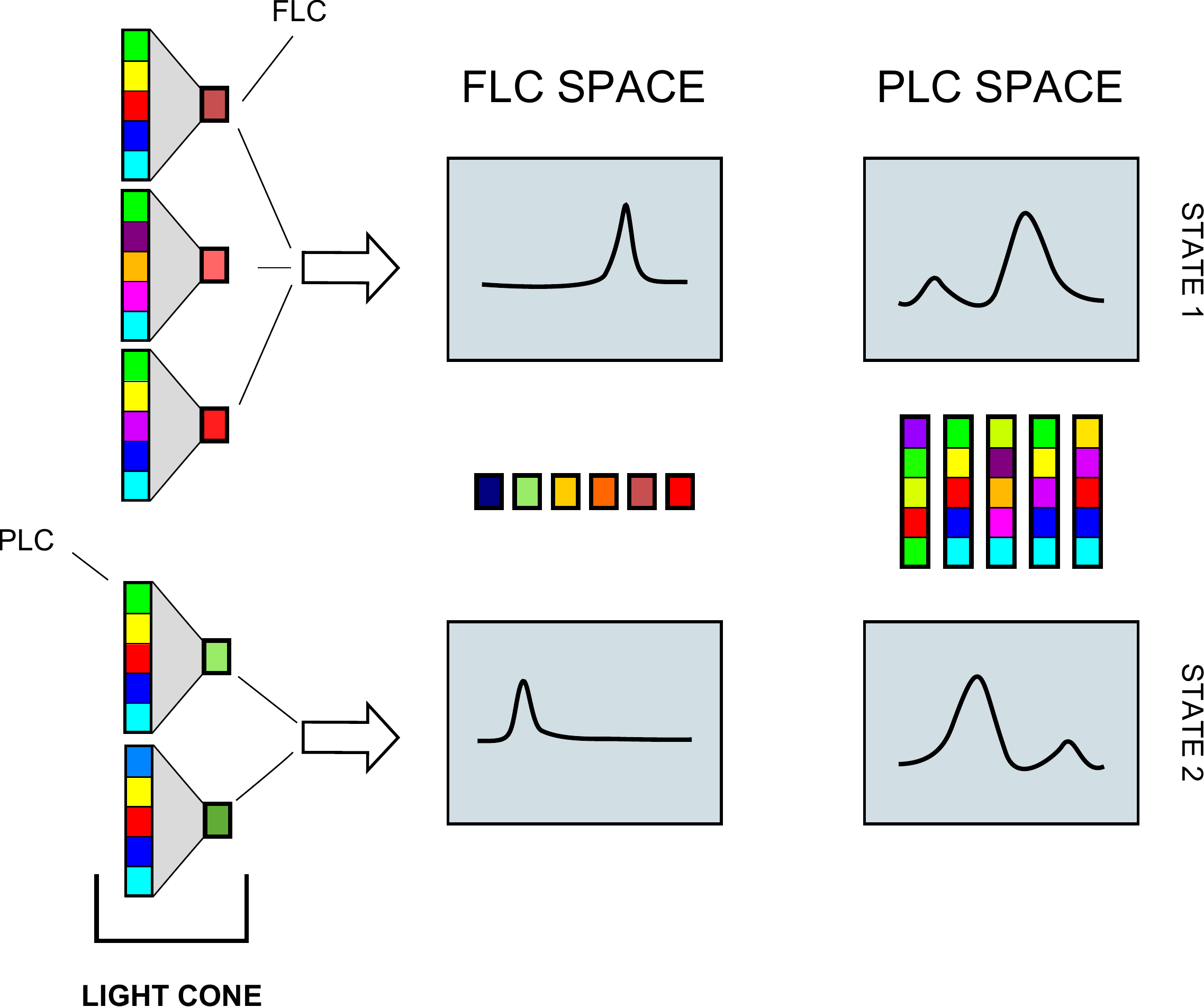}
\caption{The One Hundred Proof algorithm's input and output. Light cones are
  input, clustered using the FLCs, which results in density estimates for PLCs
  and FLCs for each state. Densities are drawn as one-dimensional for
  simplicity, but are typically multi-dimensional, continuous objects.}
\label{fig:OHP}
\end{figure}

\begin{algorithm}
\caption{One Hundred Proof}
\label{alg:OHP}
\small
\begin{algorithmic}[1]
\STATE Decompose spatio-temporal process into light cone (PLC, FLC) observation pairs.
\STATE Cluster FLCs using $k$-means++ clustering.
\STATE Map original light cone pairs to final clusters.
\end{algorithmic}
\end{algorithm}

OHP simplifies Moonshine, with a single clustering step and subsequent mapping
of light cones to clusters. The main difference is the space in which the
clustering occurs: Moonshine clusters in the space of PLCs, but OHP clusters in
the space of FLCs.  Clustering in FLC space effectively groups past light cones
by their predictive consequences, learning a geometry of our space where points
with similar futures are ``near'' each other regardless of differences in their
histories. This results in predictive states with expected
near-minimal-variance future distributions \citep{arthur2007k}, such that once
we are sure of which state a new PLC maps to, we are highly certain of what
outcome the state will generate.

To motivate this choice, imagine that all pasts map to some small set of distinct
futures, such as to the letters of a discrete finite alphabet. Given input past $\ell^-$ we want to
estimate a probability function over output $X$, so one way to do this is to group all
occurrences of future $X=x$, and use that cluster to estimate the distribution, using Bayes'
Theorem, namely,
\[
    P(X = x \mid \ell^-) \propto P(\ell^- | X = x) P(X = x).
\]
Using nonparametric density estimation over the points observed with outcome $x$, we can estimate the first quantity on the
right hand side, and taking the normalized number of member outcomes allows one to estimate the second.
This example can easily extend to continuous quantities, by clustering in the space of observed future
outcomes and substituting predictive states for the finite alphabet, which is the motivation for the OHP algorithm.

The two steps of OHP are (Algorithm \ref{alg:OHP}):
\begin{enumerate}
\item \textbf{Cluster FLCs:} After decomposing our spatio-temporal process into light cones, we cluster the FLCs using $k$-means++. The number of clusters (which will become the number of predictive states) is a user-defined parameter.
\item \textbf{Map Light Cones:} We then map the original light cones to our clusters, and produce our final predictive states, which consist of unique sets of PLCs and FLCs.
\end{enumerate}

As in the case of Moonshine, the FLCs and PLCs for each state $S_j$ are used to compute nonparametric density estimates over the space of FLCs and PLCs, providing estimators for $P(X|S_j)$ and $P(\ell^{-}| S_j)$ respectively. Algorithm~\ref{alg:OHP} outlines the process of state reconstruction for OHP.

\subsection{Predictive Distributions for Light Cone Systems}

%

Given the states reconstructed by Moonshine or OHP, we can estimate predictive
distributions as follows.  The conditional probability (or probability density)
of $X$ given PLC $\ell^-$ is obtained by mixing over the predictive states,
namely
\begin{align}
    P(X|\ell^-) &= \sum_{j=1}^{K} P(X, S_j | \ell^-) \\
                &= \sum_{j=1}^{K} P(X|S_j)P(S_j | \ell^{-}) \label{eq:dist-over-x}
\end{align}
where the second equality follows from the conditional independence of $X$ and $\ell^{-}$ given the predictive state $S_j$. The $P(S_j | \ell^{-})$ terms serve as the mixture weights, and Bayes's Theorem yields
\begin{align}
    P(S_j | \ell^{-}) &= \frac{P(\ell^{-}| S_j)P(S_j)}{P(\ell^{-})} \\
                      &= \frac{P(\ell^{-}| S_j)P(S_j)}{\sum_{k=1}^{K}P(\ell^{-}| S_k)P(S_k)}.\label{eq:conditional-of-state}
\end{align}
All of the quantities in (\ref{eq:dist-over-x}) and
(\ref{eq:conditional-of-state}) can be estimated using our reconstructed
predictive states, which are each associated with unique sets of PLCs and
FLCs. We estimate $P(S_j)$ by $N_j / N$, where $N$ is the total number of light
cone observations and $N_j$ is the number of light cones assigned to state
$S_j$. The two state-conditioned densities $P(X|S_j)$ and $P(\ell^{-}| S_j)$
are estimated using nonparametric density estimation techniques (such as kernel
density estimation) based on their associated FLCs and PLCs.  Thus we get
\begin{align}
    \widehat{P}(X|\ell^-) &= \sum_{j=1}^{K} \left(\frac{N_j \widehat{P}(\ell^{-}| S_j)}{\sum_{k=1}^{K}N_k \widehat{P}(\ell^{-}| S_k)}\right) \widehat{P}(X|S_j)
\end{align}
where $\widehat{P}(X| S_k)$ and $\widehat{P}(\ell^{-}| S_k)$ denote the
nonparametric density estimates of the two corresponding conditional densities.

When we need a point prediction of $X$, we use the conditional mean:
\begin{align}
    \mathbb{E}\left[X | \ell^{-} \right] &= \mathbb{E}\left[\mathbb{E}\left[ X | \ell^{-}, S\right] \ell^{-} \right] \\
            &= \mathbb{E}\left[\mathbb{E}\left[ X | S\right] \ell^{-} \right] \\
            &= \sum_{j=1}^{K}{P(S_j | \ell^{-}) \mathbb{E}[X|S_j]}.
\end{align}
Replacing $P(S_j | \ell^{-})$ with (\ref{eq:conditional-of-state}), plugging in
the estimated densities and probabilities, and using the mean future value for
state $S_j$ (denoted $\overline{x}_{j}$) to estimate $\mathbb{E}[X|S_j]$, we
obtain the final prediction rule
\begin{align}
    X^* &= \sum_{j=1}^{K} \left(\frac{N_j \widehat{P}(\ell^{-}| S_j)}{\sum_{k=1}^{K}N_k \widehat{P}(\ell^{-}| S_k)}\right) \overline{x}_{j} \\
        &= \sum_{j=1}^{K}w_j(\ell^-) \overline{x}_{j},
\end{align}
which is simply a suitably weighted mixture of the mean predictions for each state.

\subsection{Light Cone Linear Regression}

If only predictive regression is needed and not a full generative model, one can perform linear regression directly using light cones. Light cone linear regression uses the same light cone decomposition as the LICORS, Moonshine and OHP methods, but learns a regression rule directly from past light cones to future light cone values. This has the advantages of extremely fast prediction and good forecasting accuracy, along with simple implementation. We evaluate the performance of light cone linear regression on two real-world forecasting tasks, in \S \ref{sec:results}.

\section{Experimental Setup}
\label{sec:results}

In order to evaluate the effectiveness of light cone methods, we attempt spatio-temporal forecasting on real-world
data.

\subsection{Forecasting Task 1: Electrostatic Potentials}
For the first task, the data come from a set of experiments measuring
electrostatic potential changes in organic electronic materials
\citep{hoffmann2013asymmetric}.\footnote{Specifically, the data were collected
  using kelvin force probe microscopy to measure spatio-temporal changes in
  electrostatic charge regions on the surface of poly(3-hexyl)thiophene film.}
We learn a common set of predictive states across experiments, and do
frame-by-frame prediction on a single held-out experiment, effectively
cross-validating across experiments.

Each experiment consists of 7--10 time slices, or \emph{frames}.  Each frame is
a 256-by-256 matrix of scalar measurements, which we call \emph{pixels}, since
the data resembles video in structure.  Predictions are performed for
254-by-254 pixels in each frame after the first, which allows for each pixel to
be predicted based on a full light cone, thus excluding marginal light cones.

\subsection{Forecasting Task 2: Human Speaker Video}

For the second task, we predict the next frame of a full-resolution video from a recording of a human speaker, used in generating an intelligent avatar agent.\footnote{Used with permission from GetAbby (True Image Interactive, LLC).} In this task, we perform leave-one-frame-out predictions, cross-validating across video frames. Each frame consists of 440-by-330 pixels, of which predictions are performed on the 428-by-328 interior pixels, again excluding marginal light cones. Every fifth frame from the video is retained, and light cones are extracted from roughly one hundred skip frames. Forty-thousand light cones are subsampled for tractability. These light cones are used for cross-validation.

\subsection{Comparison Methods and Parameter Settings}

We compare the performance of predictive state reconstruction and forecasting systems with some simple baseline methods. For all light cone methods, the same set of light cones were extracted from the data, with $h_p = 1$, $h_f = 0$, and $c = 1$, resulting in PLCs of dimension $d=9$ and FLCs with dimension $d=1$. We evaluate the performance of the mixed LICORS system, implemented by the authors following \citet{goerg2013mixed}. For tractability, only twenty thousand light cones were used in training each fold for the first task, and forty thousand for the second. Kernel density estimators were used for both PLC density estimation as well as FLC density estimation, to improve predictive performance. Initialization was performed using $k$-means++ and the iteration delta was set to $0.0019$. For light cone linear regression, we use linear regression implemented in the scikit-learn package for Python \citep{scikit-learn}, version 15.2.

The simplest method we compare against is the ``predict the value from the last
frame'' method that simply takes the previous value of a pixel and uses that as
the prediction for the pixel in the current frame. The $k$-nearest neighbor
regressor takes as input a past light cone and finds the $k$-nearest PLCs in
Euclidean space, then takes the weighted average of their individual future
light cone values and outputs that as the current prediction.  Below, we report
results from the scikit-learn implementation of KNeighborsRegressor with
default parameter settings.

\subsection{Performance Metrics}

We compared performance in terms of mean-squared-error (MSE) and correlation
(Pearson $\rho$) with the ground truth. Additionally, for the three
distributional methods (mixed LICORS, Moonshine and OHP) we measured the
average per pixel log-likelihood (Avg.\ LL) of the predictions, an estimate of
the (negative) cross-entropy between the model and the truth, and the
perplexity ($2^{-\text{Avg\ LL}}$), with lower perplexity being better.  For
the distributional methods, we tested performance both for a large maximum
number of states ($K_{\text{max}} = 100$) and a small number of states ($K_{\text{max}} = 10$).

To avoid negative infinities appearing when model likelihoods are sufficiently
close to zero, we apply smoothing to the three distributional models for all
likelihood estimates mapping to zero, converting them to likelihoods of
$10^{-300}$.

\subsection{Qualitative Results}
\begin{figure*}[htbp]
\centering
\includegraphics[width=1.\linewidth]{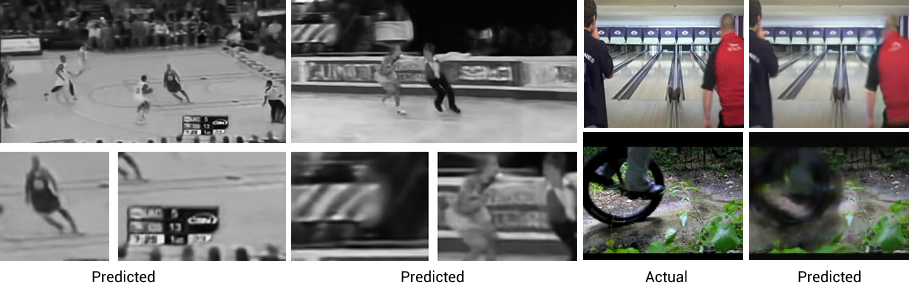}
\caption{Prediction examples of \citet{mathieu2015deep}}
\label{fig:comparison}
\end{figure*}

\begin{figure*}[htbp]
\centering
\includegraphics[width=1.\linewidth]{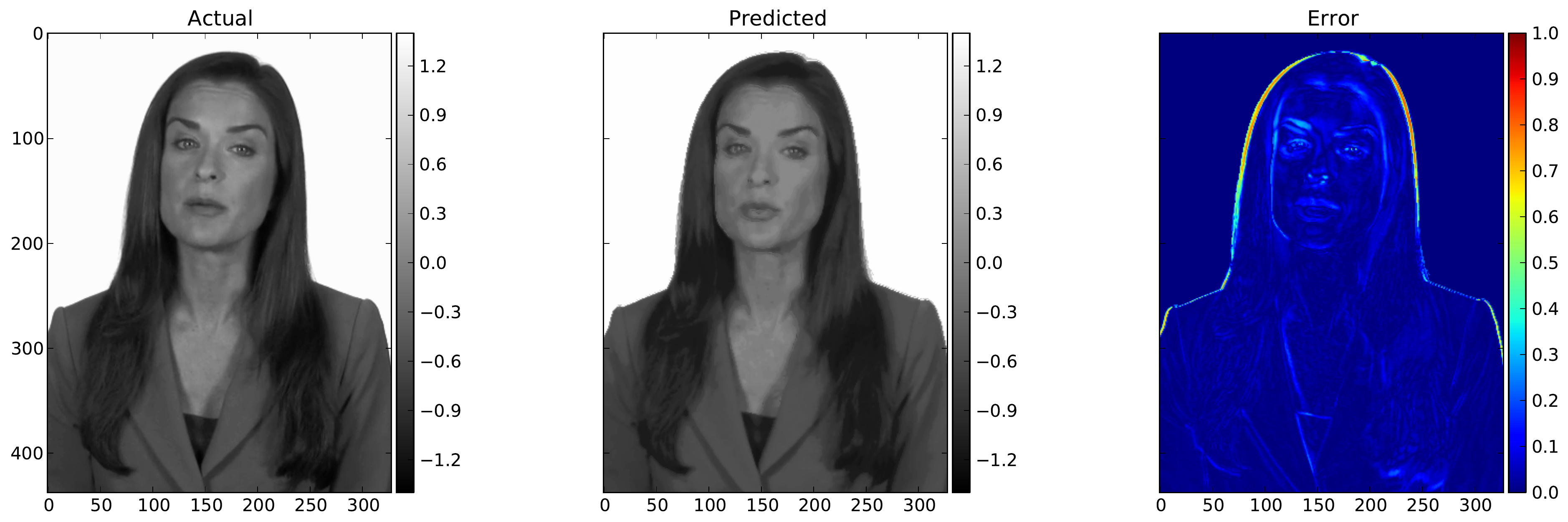}
\includegraphics[width=1.\linewidth]{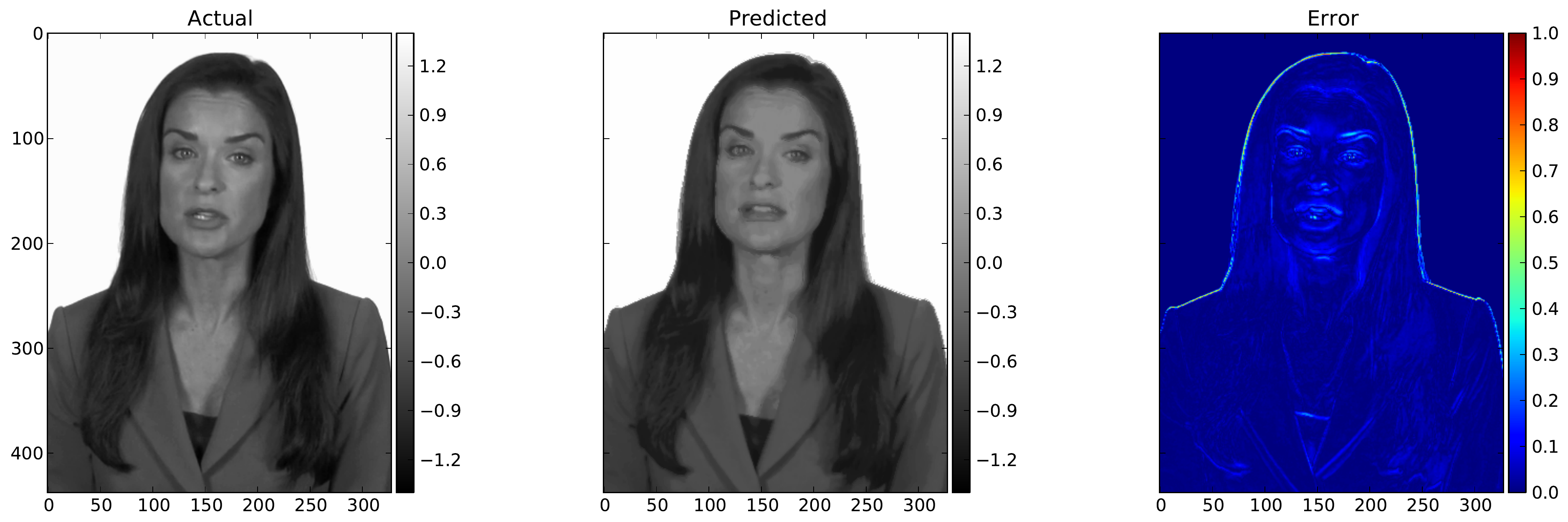}
\caption{Predicting video with mixed LICORS light cone system.}
\label{fig:abby-pred}
\end{figure*}

\begin{figure}[htbp]
\centering
\includegraphics[width=1.\linewidth]{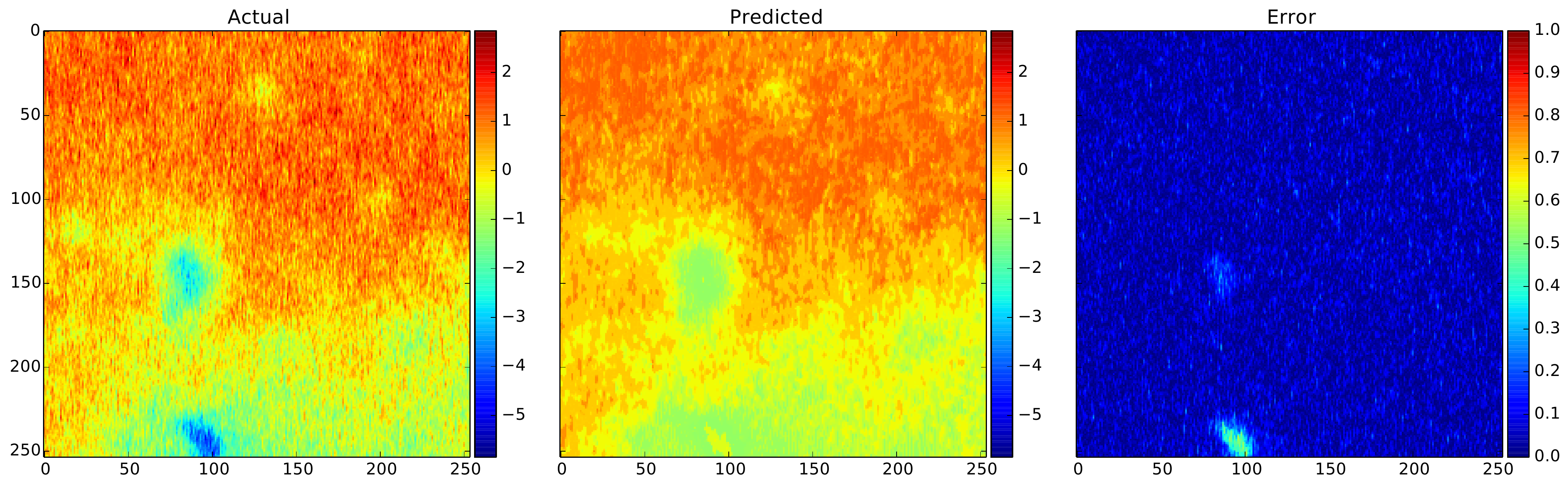}
\includegraphics[width=1.\linewidth]{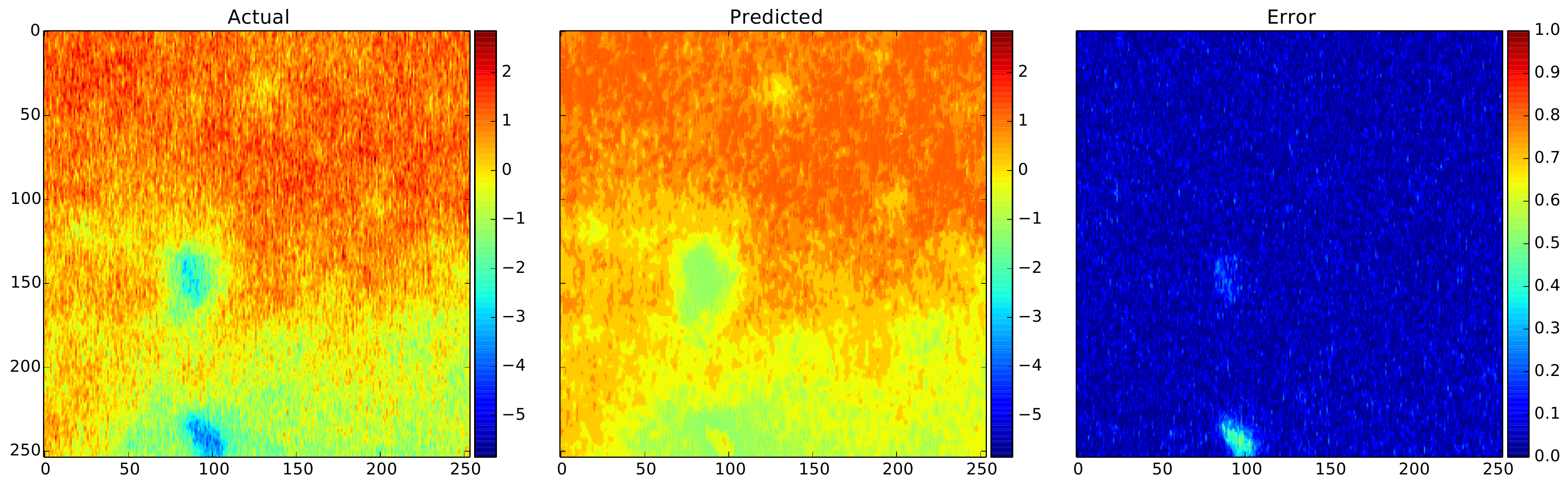}
\includegraphics[width=1.\linewidth]{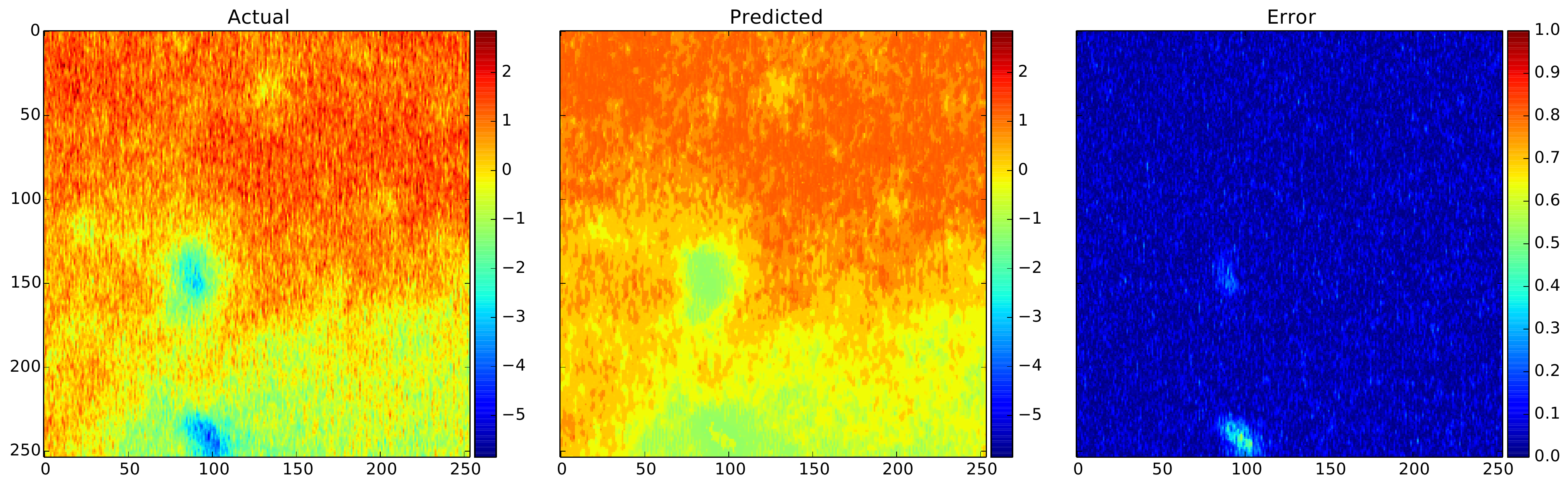}
\caption{Predicting electrostatic potentials with Moonshine.}
\label{fig:Moonshine-pred}
\end{figure}

\begin{figure}[htbp]
\centering
\includegraphics[width=1.\linewidth]{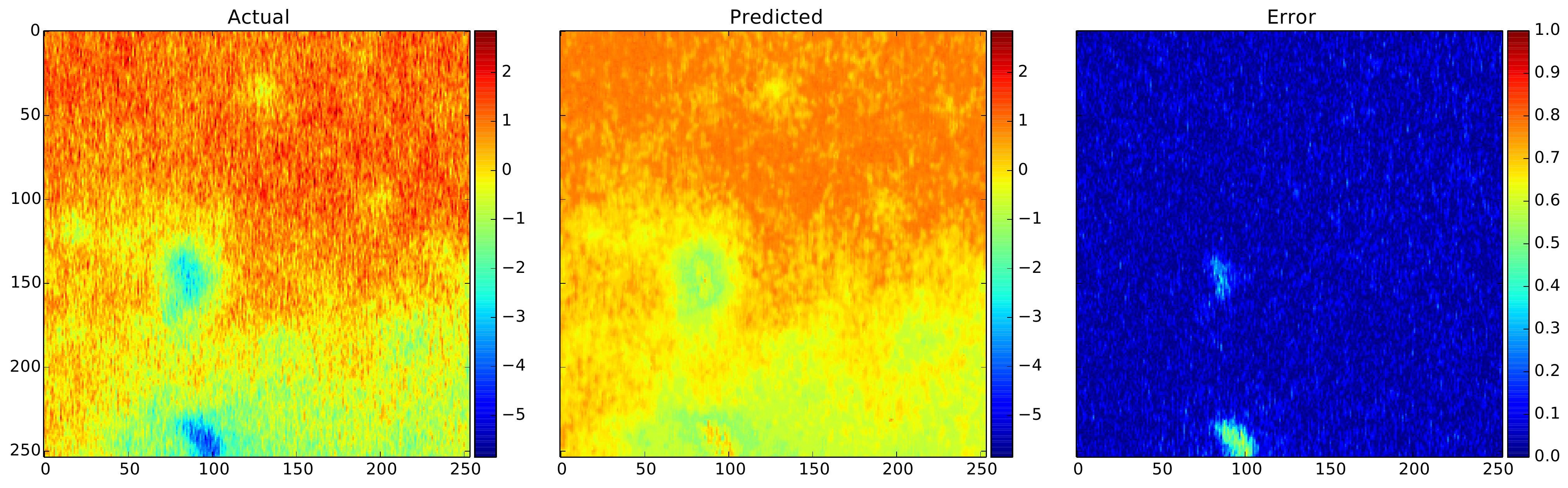}
\includegraphics[width=1.\linewidth]{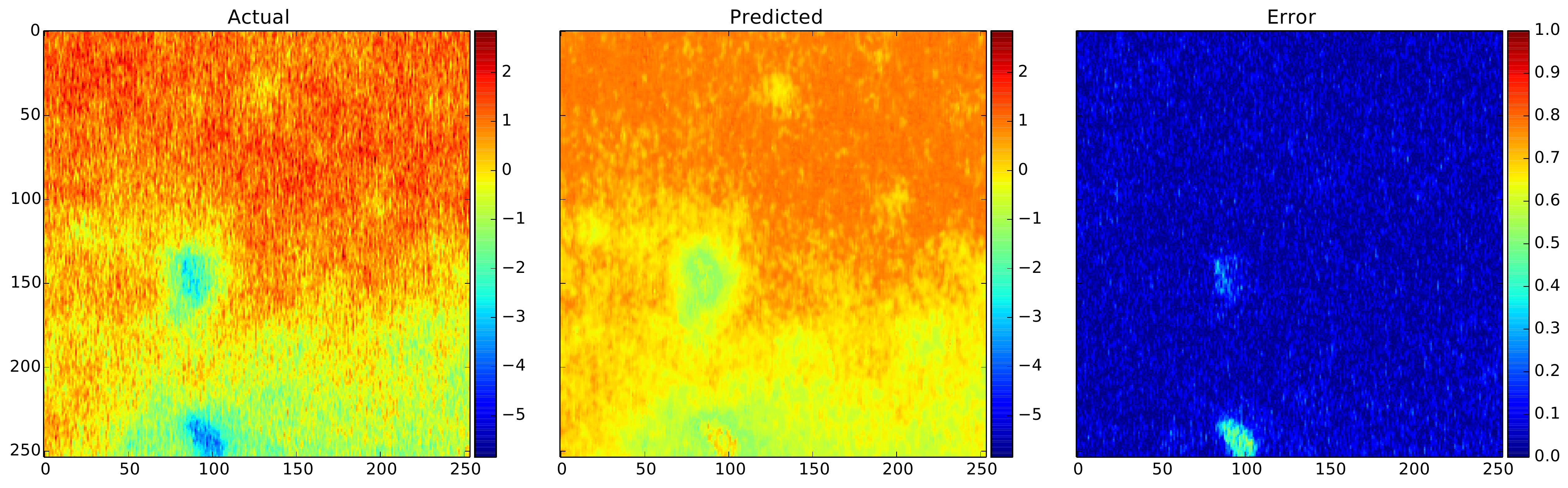}
\includegraphics[width=1.\linewidth]{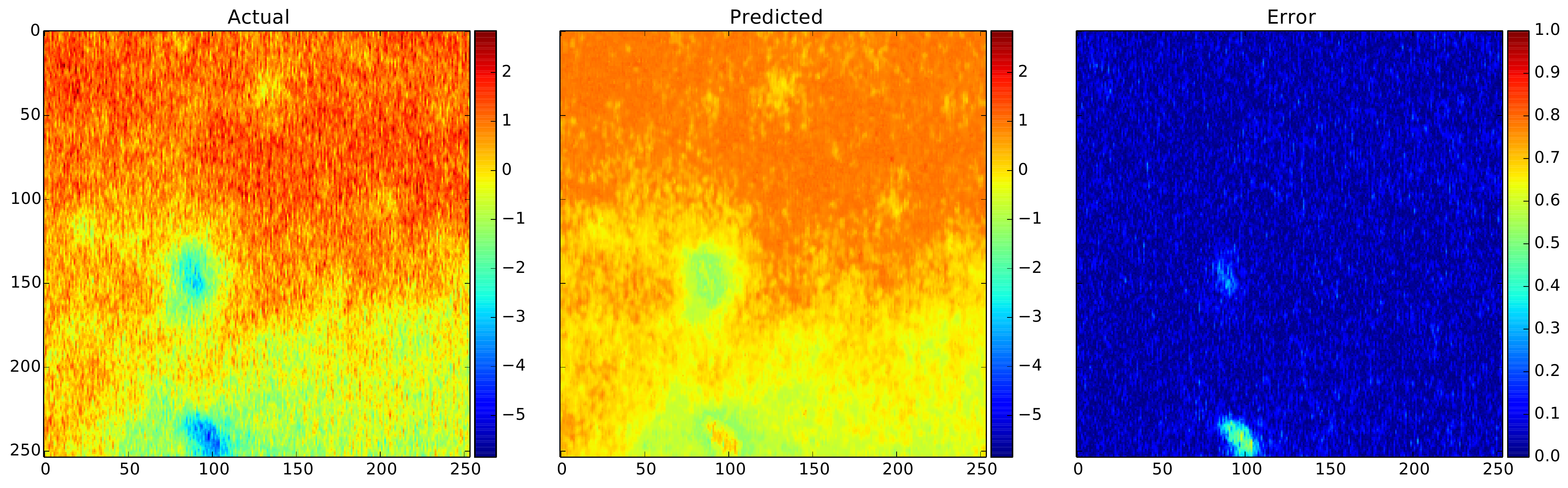}
\caption{Predicting electrostatic potentials with OHP.}
\label{fig:OHP-pred}
\end{figure}

Light cone systems compare favorably to state-of-the-art deep learning methods, such as
\citet{mathieu2015deep} (seen in Figure~\ref{fig:comparison}), which improves
on earlier work by \citet{ranzato2014video}. The amount of blurring and structural
 aberration becomes noticeable in their prediction examples, reproduced here. Compare with Figure~\ref{fig:abby-pred},
where a light cone system (mixed LICORS) is used to predict the next frame of human video. The
light cone predictions maintain strong structural consistency and minimal blurring, at the cost of
some quantization effects (due to predictive state clustering).

For the electrostatic potentials prediction task, Fig.\ \ref{fig:Moonshine-pred} and \ref{fig:OHP-pred} show
three frames of predictions each for Moonshine and OHP, respectively.  The next
frame (top to bottom) is predicted using models trained on the remaining six
experiments, given PLCs from the previous frame. Error percentage was
calculated as a proportion of the maximum dynamic range of the actual values or
predictions, namely,
\[
    err_{\text{pct}} = \frac{|t - p|}{|\max\{v : v \in T \cup P\} - \min\{v : v \in T \cup P\}|}
\]
where $T$ is the set of true testing frames, $P$ is the set of predicted
frames, $t$ is the true value at a pixel, $p$ is the predicted value of a pixel
and $|\cdot|$ is the $L_1$ norm.  Qualitatively, both methods do well,
capturing much of the changing dynamics in each frame.  The methods have
trouble representing the extreme values at the two ``hotspots'' (visible in the
error plots in the third columns), giving instead over-smoothed
predictions. Other than those extreme regions, the error residuals lack obvious
structure and are relatively small.


\subsection{Quantitative Results}

\begin{table*}[htbp]
  \centering
  \footnotesize
  \caption{Results for predicting electrostatic potentials.}
  \resizebox{1.0\linewidth}{!}{
    \begin{tabular}{lrrrrrrrr}
      \toprule
      \textbf{Method}  & $K_{\text{max}}$ & \textbf{MSE} &  \textbf{95\% CI} &  \textbf{Pearson $\rho$} &  \textbf{95\% CI} &  \textbf{Avg.\ LL} & \textbf{95\% CI} & \textbf{Perplexity} \\
      \midrule
      Future-like-the-Past  	            & $\cdot$ & 0.778 & [0.777, 0.780] & 0.615 & [0.614, 0.616]  & $\cdot$ & $\cdot$ & $\cdot$ \\
      KNN Regression              	    & $\cdot$ & 0.852 & [0.851, 0.853] & 0.506 & [0.505, 0.506]  & $\cdot$ & $\cdot$ & $\cdot$ \\
      \midrule
      Light Cone Linear Regression        & $\cdot$ & 0.607 & [0.606, 0.608] & 0.628 & [0.627, 0.628]  & $\cdot$ & $\cdot$ & $\cdot$ \\
      Mixed LICORS                        & 100 & \textbf{0.569} & \textbf{[0.567, 0.571]} & \textbf{0.663} & \textbf{[0.661, 0.665]} & -1.034 & [-1.110, -0.964] & 2.052 \\
      Moonshine           	            & 100 & \textbf{0.570} & \textbf{[0.569, 0.572]} & 0.656 & [0.655, 0.657] & \textbf{-0.672} & \textbf{[-0.727, -0.617]} & \textbf{1.593} \\
      One Hundred Proof           	    & 100 & 0.592 & [0.591, 0.593] & 0.641 & [0.640, 0.642] & -1.724 & [-2.127, -1.321] & 3.303 \\
      \midrule
      Mixed LICORS                        & 10  & \textbf{0.566} & \textbf{[0.565, 0.567]} & \textbf{0.668} &  \textbf{[0.667, 0.669]} & -1.022 & [-1.096, -0.947] & 2.030 \\
      Moonshine           	            & 10  & 0.609 & [0.605, 0.613] & 0.625 &  [0.622, 0.628] & -0.722 & \textbf{[-0.767, -0.678]} & 1.650 \\
      One Hundred Proof           	    & 10  & 0.597 & [0.595, 0.598] & 0.648 &  [0.646, 0.649] & \textbf{-0.682} & \textbf{[-0.757, -0.608]} & \textbf{1.605} \\
      \bottomrule
    \end{tabular}
  }
  \label{tab:real-world-regression}
\end{table*}

Table~\ref{tab:real-world-regression} shows how well each method did at
predicting electrostatic potentials (Task 1).  Mixed LICORS and Moonshine have
the lowest MSE, with 95\% confidence intervals disjoint from the intervals of
other methods. Mixed LICORS also has the highest (Pearson) correlation with the
true values. Lastly, of the generative methods (i.e., mixed LICORS, Moonshine
and One Hundred Proof), Moonshine and OHP have the highest average
log-likelihood and lowest perplexity. Thus, mixed LICORS and Moonshine provide
the best overall performance on the dataset.

Restricting ourselves to the generative methods for a compact number of states
($K_{\text{max}} = 10$), mixed LICORS has the lowest average MSE, while
Moonshine and One Hundred Proof have the best probabilistic performance, giving
the highest likelihoods and lowest perplexities for the data.

    \begin{table*}[htbp]
    \centering
    \footnotesize
    \caption{Results for predicting video of human speakers.}
    \resizebox{1.0\linewidth}{!}{
    \begin{tabular}{lrrrrrrrr}
    \toprule
    \textbf{Method}  & $K_{\text{max}}$ & \textbf{MSE} &  \textbf{95\% CI} &  \textbf{Pearson $\rho$} &  \textbf{95\% CI} &  \textbf{Avg.\ LL} & \textbf{95\% CI} & \textbf{Perplexity} \\
    \midrule
    Future-like-the-Past  	            & $\cdot$ & 0.031 & [0.031, 0.031] & 0.984 & [0.984, 0.984]  & $\cdot$ & $\cdot$ & $\cdot$ \\
    KNN Regression              	    & $\cdot$ & 0.033 & [0.033, 0.033] & 0.984 & [0.984, 0.984]  & $\cdot$ & $\cdot$ & $\cdot$ \\
    \midrule
    Light Cone Linear Regression        & $\cdot$ & \textbf{0.028} & \textbf{[0.028, 0.028]} & \textbf{0.986} & \textbf{[0.986, 0.0986]} & $\cdot$ & $\cdot$ & $\cdot$ \\
    Mixed LICORS                        & 100 & 0.038 & [0.038, 0.038] & 0.981 & [0.981, 0.981] & 0.102 & [0.099, 0.105] & 0.932 \\
    Moonshine           	            & 100 & 0.039 & [0.039, 0.039] & 0.981 & [0.981, 0.981] & \textbf{0.925} & \textbf{[0.874, 0.976]} & \textbf{0.527} \\
    One Hundred Proof           	    & 100 & 1.060 & [0.460, 1.659] & 0.911 & [0.871, 0.952] & -6.48 & [-8.025, -4.948] & 89.641 \\
    \bottomrule
    \end{tabular}
    }
    \label{tab:abby}
\end{table*}

Table~\ref{tab:abby} gives the results from video prediction (Task 2). Light
cone linear regression has the strongest overall performance, with low error
and high correlation to the ground truth.  However, the strong temporal
consistency of this dataset allows even the FLTP method to perform remarkably
well, outperforming the predictive state light cone methods. While forecasting
is relatively easy for this task, being able to estimate a likelihood model for
such data gives the predictive state methods an edge over pure regression methods.

\section{Discussion}
\label{sec:discussion}

In this manuscript, we have tested an existing light cone method (mixed LICORS), qualitatively 
comparing it to deep learning methods, and introduced three new light cone methods (light cone linear regression, Moonshine, OHP). The two 
latter predictive state methods are successive approximations of the approach used by mixed LICORS, with OHP pushing the 
limit of how simplified we could make the approximation. OHP is demonstrated to be one approximation too far, since the 
increased simplification comes at the cost of degraded performance.

On the first real-world spatio-temporal regression task, we find that the three
LICORS-inspired methods (mixed LICORS, Moonshine and One Hundred Proof) are
able to accurately forecast the changing dynamics of the underlying
spatio-temporal system. Furthermore, being generative methods, they can be used
to compute the likelihood of spatio-temporal data. Moonshine and One Hundred
Proof (OHP) are conceptually simple, easy to implement alternatives to the full
mixed LICORS system, which give comparable performance for likelihood
estimation and forecasting on this task. Although OHP is the simplest method,
it fails to perform well in some contexts, such as the second video prediction
task, showing a trade-off between method simplicity and forecasting
performance.

Light cone linear regression is a fast and simple method, and is able to
perform well on both prediction tasks. It does not estimate likelihoods over
data as do the other predictive-state methods, but moving to generalized linear
models would allow this. It shows the effectiveness of light cone decompositions
and remains a useful approach.

Overall, the best performance on all tasks was achieved or shared by the three new methods, with Moonshine 
having the best probabilistic modeling performance on both tasks, light cone linear regression having the best 
forecasting performance on the second task, and OHP having good modeling performance under the constrained setting 
of limited number of states. Moonshine has better probabilistic modeling performance than mixed LICORS on these tasks, and has statistically indistinguishable forecasting capability (see Tables~\ref{tab:real-world-regression} (100 state case) and \ref{tab:abby}). While it might be argued that the improved performance was not improved \emph{enough}, 
we have to remind ourselves that these are approximations -- that they improve performance at all is surprising. 

Although OHP does have limited forecasting ability, it does manage to model at least one of the datasets well, showing that its simplified form is not entirely without merit. This, at very least, shows when approximations become too simplified to accomplish complex tasks. Negative results are important, especially when detecting boundaries.

\section{Theoretical Results}
\label{sec:theory}

We state a result for light cone predictive state systems,
with proof given in Appendix I.

We wish to bound the error of our estimated distribution over futures given
pasts, namely, the error of $\widehat{P}(X|\ell^-)$. For a fixed random sample
of data, let $P^*(X|\ell^-)$ denote the optimal estimate for $P(X|\ell^-)$
constructable from the sample. We begin by noting 
\begin{align*}
    |\widehat{P}(X|\ell^-) - P(X|\ell^-)| &\leq |\widehat{P}(X|\ell^-) - P^*(X|\ell^-)| +\\
                                          &\phantom{+}\;\; |P^*(X|\ell^-) - P(X|\ell^-)|.
\end{align*}

The second summand on the right-hand side is the gap between the
optimal estimate and truth, which we assume to shrink in probability with the
sample size (as in \citealt{goerg2012licors}).  We focus on first term, which
is the gap between our light-cone based nonparametric estimator and the optimal
estimate.  For this quantity, we state our main result:
\begin{thm}
  For a fixed data sample of size $N$, let $P^*(X|\ell^-)$ denote the optimal
  estimator based on that sample and $\widehat{P}(X|\ell^-)$ be the light cone
  estimator based on the same sample. Let $\widehat{P}(X|S_j)$ be bounded by a
  constant $M$ for all $X, j$. If 
  \[
    |\widehat{P}(S_j|\ell^-) - P^*(S_j|\ell^-)|\rightarrow 0
  \]
  for all $j$, then for any $X$, $\epsilon > 0$, $\delta > 0$, and sufficiently large $N$,
   {\footnotesize
\[
    \PP\left(|\widehat{P}(X|\ell^-) - P^*(X|\ell^-)| > \epsilon\right) \leq 2\exp\left\{-\frac{2(1+N^{*}_{s})^2(\epsilon - \delta)^2}{N K_h(0)^2}\right\},
\]
   }
where $N^*_{s}$ is the (smallest) sum of weights for the predictive states and $K_h(\cdot)$ is a bandwidth $h$ kernel.
\end{thm}
{\em Proof sketch} (see appendix for details): For the quantity
$|\widehat{P}(X|\ell^-) - P^*(X|\ell^-)|$, we first mix over states, and use
the chain rule to condition.  Then we add and subtract
$\widehat{P}(X|S_j)P^*(S_j|\ell^-)$, and split the sum into two parts, one
multiplied by $P^*(S_j|\ell^-)$ and the other multiplied by
$\widehat{P}(X|S_j)$. By the assumptions stated, the second sum is bounded and
decreasing to zero, so that for sufficiently large $N$ it is smaller than any
$\delta >0$. The first sum is less than $\max_{j}|\widehat{P}(X|S_j) -
P^*(X|S_j)|$, which we bound with high probability, using a Hoeffding bound for
dependant data \citep{van-de-Geer-on-Hoeffding-for-dependent}. The result
follows directly from application of the Hoeffding bound.

\section{Related and Future Work}\label{sec:RELATED-FUTURE}

\subsection{Related Work}

Our debt to \citet{goerg2012licors,goerg2013mixed} needs no elaboration.  We
share the same general framework, but aim at simpler algorithms, even if it
costs some predictive power. The work on LICORS grows out of earlier work on
predictive Markovian representations of non-Markovian time series
\citep{Knight-predictive-view,Inferring-stat-compl,CMPPSS,CSSR-for-UAI}, whose
transfer to spatio-temporal data was originally aimed at unsupervised pattern
analysis in natural systems \citep{QSO-in-PRL,Automatic-Filters}; our
qualitative results suggest Moonshine and OHP remain suitable for this, as well
as for prediction.  The formalism used in this line of work is mathematically
equivalent to the ``predictive representations of state'' introduced by
\citet{predictive-representations-of-state}, and lately the focus of much
interest in conjunction with spectral estimation methods
\citep{Boots-et-al-online-spectral-learning}.  Both formalisms are also
equivalent to observable operator models \citep{Jaeger-operator-models} and to
``sufficient posterior'' representations \citep{Langford-Salakhutdinov-Zhang};
our approach may suggest new estimation algorithms within those formalisms.

\subsection{Future Work}

\begin{figure}[htbp]
\centering
\includegraphics[width=1.\linewidth]{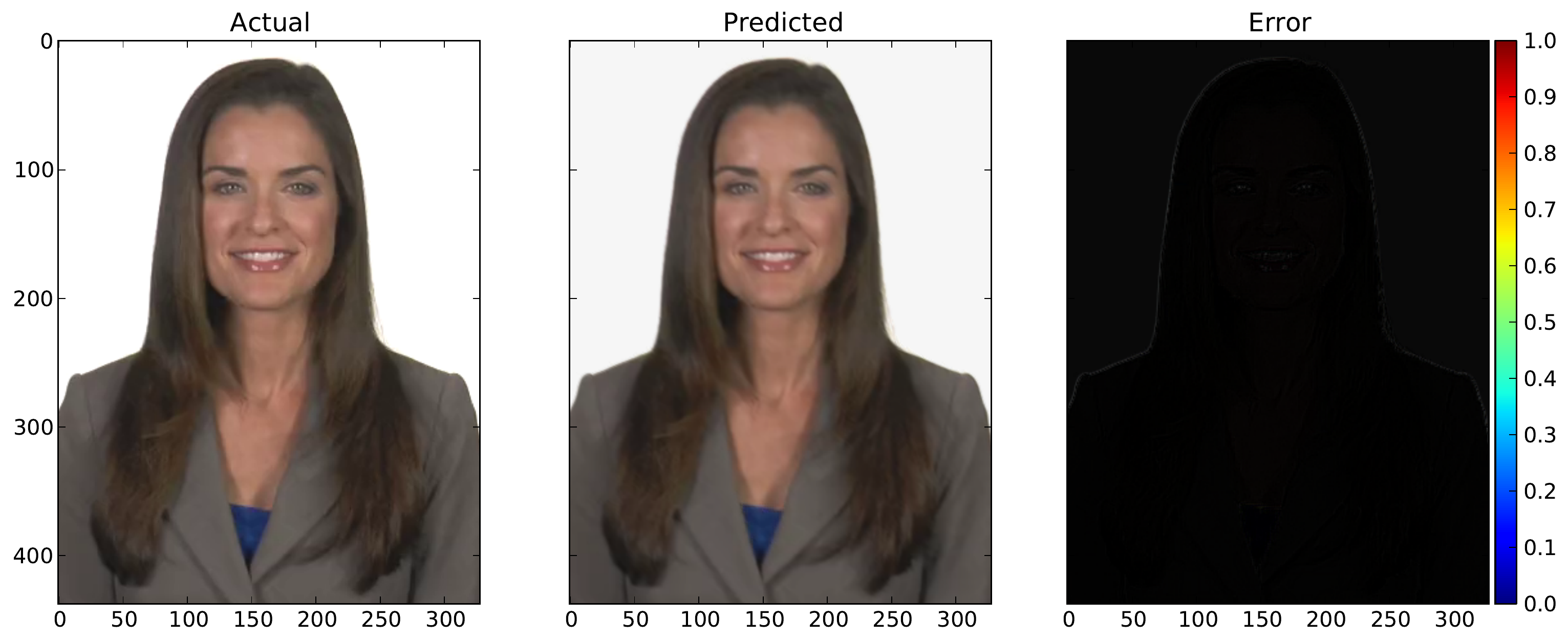}
\caption{Color prediction of human film data, using mixed LICORS light cone forecasting.}
\label{fig:iva-pred}
\end{figure}

Light cone methods, such as the three described here, hold promise for the prediction of dynamical systems. Given the
flexibility and generality of light cone decompositions, one can easily extend such methods to handle full-color video (e.g., Figure~\ref{fig:iva-pred}),
and Kinect\texttrademark-sensor depth video. These applications are the focus of current and future research.

The ``rate limiting step'' for approximate light cone methods like Moonshine
and OHP is the speed of nonparametric density estimation. Methods that scale
poorly in the number of observations are of limited use. Towards that end,
future research into fast approximate nonparametric density estimation will
improve the computational efficiency of the methods presented.

The theoretical properties of the two predictive state methods will be further
explored in a future paper, especially with regard to the trade-offs in their
approximation to what LICORS or mixed LICORS would do, and the influence of the
new algorithms' internal randomness.

\section{Conclusion}\label{sec:CONCLUSION}

Faced with the task of learning to accurately model video-like data, we
explore the strengths and drawbacks of light cone decomposition methods and
propose new simplified nonparametric predictive state methods inspired by the mixed
LICORS \citep{goerg2013mixed} algorithm. The methods, Moonshine and One Hundred
Proof, do not require costly iterative EM training or the memory intensive
formation of an explicit $N \times K$ matrix, yet retain the generative
modeling capabilities and are competitive in predictive performance to the original mixed
LICORS method. The methods are shown to perform well on one real-world data task,
effectively capturing spatio-temporal structure and outperforming baseline methods,
while a light cone version of linear regression performs well on the remaining task.
Overall, we see that light cone decompositions of complex spatio-temporal data can open
opportunities to tractably estimate probability densities and accurately forecast the
changing systems. By introducing simplified versions of light cone algorithms, we hope to
encourage further exploration and application of this general technique.

\bibliographystyle{abbrvnat}
\IEEEtriggeratref{13}
{\small
\bibliography{references,locusts}}

\begin{thebibliography}{27}
\providecommand{\natexlab}[1]{#1}
\providecommand{\url}[1]{\texttt{#1}}
\expandafter\ifx\csname urlstyle\endcsname\relax
  \providecommand{\doi}[1]{doi: #1}\else
  \providecommand{\doi}{doi: \begingroup \urlstyle{rm}\Url}\fi

\bibitem[Arthur and Vassilvitskii(2007)]{arthur2007k}
D.~Arthur and S.~Vassilvitskii.
\newblock k-means++: The advantages of careful seeding.
\newblock In \emph{Proceedings of the eighteenth annual ACM-SIAM symposium on
  Discrete algorithms}, pages 1027--1035. Society for Industrial and Applied
  Mathematics, 2007.

\bibitem[Boots and Gordon(2011)]{Boots-et-al-online-spectral-learning}
B.~Boots and G.~Gordon.
\newblock An online spectral learning algorithm for partially observable
  nonlinear dynamical systems.
\newblock In W.~Burgard and D.~Roth, editors, \emph{Proceedings of the 25th
  National Conference on Artificial Intelligence (AAAI-2011)}, pages 293--300,
  Menlo Park, California, 2011. AAAI.

\bibitem[Crutchfield and Young(1989)]{Inferring-stat-compl}
J.~P. Crutchfield and K.~Young.
\newblock Inferring statistical complexity.
\newblock \emph{Physical Review Letters}, 63:\penalty0 105--108, 1989.

\bibitem[Ester et~al.(1996)Ester, Kriegel, Sander, and Xu]{DBSCAN}
M.~Ester, H.-P. Kriegel, J.~Sander, and X.~Xu.
\newblock A density-based algorithm for discovering clusters in large spatial
  databases with noise.
\newblock In \emph{KDD}, volume~96, pages 226--231, 1996.

\bibitem[Goerg and Shalizi(2012)]{goerg2012licors}
G.~M. Goerg and C.~R. Shalizi.
\newblock {L}{I}{C}{O}{R}{S}: {L}ight cone reconstruction of states for
  non-parametric forecasting of spatio-temporal systems.
\newblock \emph{arXiv preprint arXiv:1206.2398}, 2012.

\bibitem[Goerg and Shalizi(2013)]{goerg2013mixed}
G.~M. Goerg and C.~R. Shalizi.
\newblock Mixed {L}{I}{C}{O}{R}{S}: A nonparametric algorithm for predictive
  state reconstruction.
\newblock In \emph{Proceedings of the Sixteenth International Conference on
  Artificial Intelligence and Statistics}, pages 289--297, 2013.

\bibitem[Gokcay and Principe(2002)]{982897}
E.~Gokcay and J.~Principe.
\newblock Information theoretic clustering.
\newblock \emph{Pattern Analysis and Machine Intelligence, IEEE Transactions
  on}, 24\penalty0 (2):\penalty0 158--171, Feb 2002.
\newblock ISSN 0162-8828.
\newblock \doi{10.1109/34.982897}.

\bibitem[Hoffmann et~al.(2013)Hoffmann, Gagorik, Chen, and
  Hutchison]{hoffmann2013asymmetric}
P.~B. Hoffmann, A.~G. Gagorik, X.~Chen, and G.~R. Hutchison.
\newblock Asymmetric surface potential energy distributions in organic
  electronic materials via kelvin probe force microscopy.
\newblock \emph{The Journal of Physical Chemistry C}, 117\penalty0
  (36):\penalty0 18367--18374, 2013.

\bibitem[Jaeger(2000)]{Jaeger-operator-models}
H.~Jaeger.
\newblock Observable operator models for discrete stochastic time series.
\newblock \emph{Neural Computation}, 12:\penalty0 1371--1398, 2000.
\newblock \doi{10.1162/089976600300015411}.

\bibitem[Knight(1975)]{Knight-predictive-view}
F.~B. Knight.
\newblock A predictive view of continuous time processes.
\newblock \emph{Annals of Probability}, 3:\penalty0 573--596, 1975.

\bibitem[Kulhav{\'y}(1996)]{Kulhavy-recursive}
R.~Kulhav{\'y}.
\newblock \emph{Recursive Nonlinear Estimation: A Geometric Approach}, volume
  216 of \emph{Lecture Notes in Control and Information Sciences}.
\newblock Springer-Verlag, Berlin, 1996.
\newblock pp. 115.

\bibitem[Langford et~al.(2009)Langford, Salakhutdinov, and
  Zhang]{Langford-Salakhutdinov-Zhang}
J.~Langford, R.~Salakhutdinov, and T.~Zhang.
\newblock Learning nonlinear dynamic models.
\newblock In A.~Danyluk, L.~Bottou, and M.~Littman, editors, \emph{Proceedings
  of the 26th Annual International Conference on Machine Learning [ICML 2009]},
  pages 593--600, New York, 2009. Association for Computing Machinery.

\bibitem[Littman et~al.(2002)Littman, Sutton, and
  Singh]{predictive-representations-of-state}
M.~L. Littman, R.~S. Sutton, and S.~Singh.
\newblock Predictive representations of state.
\newblock In T.~G. Dietterich, S.~Becker, and Z.~Ghahramani, editors,
  \emph{Advances in Neural Information Processing Systems 14 (NIPS 2001)},
  pages 1555--1561, Cambridge, Massachusetts, 2002. MIT Press.

\bibitem[Mathieu et~al.(2015)Mathieu, Couprie, and LeCun]{mathieu2015deep}
M.~Mathieu, C.~Couprie, and Y.~LeCun.
\newblock Deep multi-scale video prediction beyond mean square error.
\newblock \emph{arXiv preprint arXiv:1511.05440}, 2015.

\bibitem[Parlitz and Merkwirth(2000)]{Parlitz-Merkwirth-local-states}
U.~Parlitz and C.~Merkwirth.
\newblock Prediction of spatiotemporal time series based on reconstructed local
  states.
\newblock \emph{Physical Review Letters}, 84:\penalty0 1890--1893, 2000.

\bibitem[Parzen(1962)]{KDE2}
E.~Parzen.
\newblock On estimation of a probability density function and mode.
\newblock \emph{The annals of mathematical statistics}, pages 1065--1076, 1962.

\bibitem[Pedregosa et~al.(2011)Pedregosa, Varoquaux, Gramfort, Michel, Thirion,
  Grisel, Blondel, Prettenhofer, Weiss, Dubourg, Vanderplas, Passos,
  Cournapeau, Brucher, Perrot, and Duchesnay]{scikit-learn}
F.~Pedregosa, G.~Varoquaux, A.~Gramfort, V.~Michel, B.~Thirion, O.~Grisel,
  M.~Blondel, P.~Prettenhofer, R.~Weiss, V.~Dubourg, J.~Vanderplas, A.~Passos,
  D.~Cournapeau, M.~Brucher, M.~Perrot, and E.~Duchesnay.
\newblock Scikit-learn: Machine learning in {P}ython.
\newblock \emph{Journal of Machine Learning Research}, 12:\penalty0 2825--2830,
  2011.

\bibitem[Rabiner(1989)]{rabiner1989tutorial}
L.~Rabiner.
\newblock A tutorial on hidden markov models and selected applications in
  speech recognition.
\newblock \emph{Proceedings of the IEEE}, 77\penalty0 (2):\penalty0 257--286,
  1989.

\bibitem[Ranzato et~al.(2014)Ranzato, Szlam, Bruna, Mathieu, Collobert, and
  Chopra]{ranzato2014video}
M.~Ranzato, A.~Szlam, J.~Bruna, M.~Mathieu, R.~Collobert, and S.~Chopra.
\newblock Video (language) modeling: a baseline for generative models of
  natural videos.
\newblock \emph{arXiv preprint arXiv:1412.6604}, 2014.

\bibitem[Rosenblatt et~al.(1956)]{KDE1}
M.~Rosenblatt et~al.
\newblock Remarks on some nonparametric estimates of a density function.
\newblock \emph{The Annals of Mathematical Statistics}, 27\penalty0
  (3):\penalty0 832--837, 1956.

\bibitem[Shalizi and Crutchfield(2001)]{CMPPSS}
C.~R. Shalizi and J.~P. Crutchfield.
\newblock Computational mechanics: Pattern and prediction, structure and
  simplicity.
\newblock \emph{Journal of Statistical Physics}, 104:\penalty0 817--879, 2001.

\bibitem[Shalizi and Klinkner(2004)]{CSSR-for-UAI}
C.~R. Shalizi and K.~L. Klinkner.
\newblock Blind construction of optimal nonlinear recursive predictors for
  discrete sequences.
\newblock In M.~Chickering and J.~Y. Halpern, editors, \emph{Uncertainty in
  Artificial Intelligence: Proceedings of the Twentieth Conference (UAI 2004)},
  pages 504--511, Arlington, Virginia, 2004. AUAI Press.

\bibitem[Shalizi et~al.(2004)Shalizi, Klinkner, and Haslinger]{QSO-in-PRL}
C.~R. Shalizi, K.~L. Klinkner, and R.~Haslinger.
\newblock Quantifying self-organization with optimal predictors.
\newblock \emph{Physical Review Letters}, 93:\penalty0 118701, 2004.
\newblock \doi{10.1103/PhysRevLett.93.118701}.

\bibitem[Shalizi et~al.(2006)Shalizi, Haslinger, Rouquier, Klinkner, and
  Moore]{Automatic-Filters}
C.~R. Shalizi, R.~Haslinger, J.-B. Rouquier, K.~L. Klinkner, and C.~Moore.
\newblock Automatic filters for the detection of coherent structure in
  spatiotemporal systems.
\newblock \emph{Physical Review E}, 73:\penalty0 036104, 2006.

\bibitem[van~de Geer(2002)]{van-de-Geer-on-Hoeffding-for-dependent}
S.~A. van~de Geer.
\newblock On {Hoeffding's} inequality for dependent random variables.
\newblock In H.~Dehling, T.~Mikosch, and M.~S{\/o}rensen, editors,
  \emph{Empirical Process Techniques for Dependent Data}, pages 161--169.
  Birkh{\"a}user, Boston, 2002.

\bibitem[Zahn(1971)]{zahn1971graph}
C.~T. Zahn.
\newblock Graph-theoretical methods for detecting and describing gestalt
  clusters.
\newblock \emph{Computers, IEEE Transactions on}, 100\penalty0 (1):\penalty0
  68--86, 1971.

\bibitem[Zhao et~al.(2015)Zhao, Poupart, Zhang, and Lysy]{zhao2015sof}
H.~Zhao, P.~Poupart, Y.~Zhang, and M.~Lysy.
\newblock Sof: Soft-cluster matrix factorization for probabilistic clustering.
\newblock In \emph{Proceedings of the Twenty-Ninth AAAI Conference on
  Artificial Intelligence}, AAAI 2015, 2015.

\end{thebibliography}

\newpage
\onecolumn
\section*{Appendix I: Proofs}

\setcounter{thm}{0}
\setcounter{lemma}{0}

\begin{lemma}\label{lem:1}
Let $f_j^*(\cdot)$ denote the density for state $j$ under the true assignment matrix $W^*$ and let ${N}^*_j = \sum_{l=1}^{N}w_{lj}$. Given an isolated change in $\epsilon$ in the weight $w_{ij}$, the difference between density estimate $\hat{f}_j(\cdot)$ and $f_j^*(\cdot)$ is bound by
\[
    |\hat{f}_j(x) - f_j^*(x)| \leq \left|\frac{\epsilon}{{N}^*_j + \epsilon}K_h(0)\right|.
\]
\end{lemma}

\begin{proof}
\begin{align}
    |\hat{f}_j(x) - f_j^*(x)| &= \left|\left[\sum_{k\not=i}^{N} \frac{w_{kj}}{\epsilon + \sum_{l=1}^{N}w_{lj}}K_h(\|x_k - x\|) + \frac{\epsilon + w_{ij}}{\epsilon + \sum_{l=1}^{N}w_{lj}}K_h(\|x_i - x\|)\right] - \sum_{k=1}^{N} \frac{w_{kj}}{\sum_{l=1}^{N}w_{lj}}K_h(\|x_k - x\|)\right|\\
        &= \left|\left[\sum_{k\not=i}^{N} \frac{w_{kj}}{\epsilon + {N}^*_j}K_h(\|x_k - x\|) + \frac{\epsilon + w_{ij}}{\epsilon + {N}^*_j}K_h(\|x_i - x\|)\right] - \sum_{k=1}^{N} \frac{w_{kj}}{{N}^*_j}K_h(\|x_k - x\|)\right|\\
        &= \left|\left[\sum_{k\not=i}^{N} \frac{{N}^*_jw_{kj}}{{N}^*_j(\epsilon + {N}^*_j)}K_h(\|x_k - x\|) + \frac{{N}^*_j(\epsilon + w_{ij})}{{N}^*_j(\epsilon + {N}^*_j)}K_h(\|x_i - x\|)\right] - \sum_{k=1}^{N} \frac{(\epsilon + {N}^*_j)w_{kj}}{{N}^*_j(\epsilon + {N}^*_j)}K_h(\|x_k - x\|) \right|\\
        &= \left|\frac{\epsilon}{\left({N}_j^{*}\right)^2 + \epsilon{N}^*_j}\left[{N}^*_jK_h(\|x_i - x\|) - \sum_{k=1}^{N}w_{kj}K_h(\|x_k - x\|)\right]\right|\\
        &= \left|\frac{\epsilon}{\left({N}_j^{*}\right)^2 + \epsilon{N}^*_j}\left[\sum_{k=1}^{N}w_{kj}K_h(\|x_i - x\|) - \sum_{k=1}^{N}w_{kj}K_h(\|x_k - x\|)\right]\right| \\
        &= \left|\frac{\epsilon}{\left({N}_j^{*}\right)^2 + \epsilon{N}^*_j}\sum_{k=1}^{N}w_{kj}\left[K_h(\|x_i - x\|) - K_h(\|x_k - x\|)\right]\right|.
\end{align}
Furthermore, we can bound this quantity by
\begin{align}
      \left|\frac{\epsilon}{\left({N}_j^{*}\right)^2 + \epsilon{N}^*_j}\sum_{k=1}^{N}w_{kj}\left[K_h(\|x_i - x\|) - K_h(\|x_k - x\|)\right]\right| &= \left|\frac{\epsilon}{{N}^*_j + \epsilon}\sum_{k=1}^{N}\frac{w_{kj}}{{N}^*_j}\left[K_h(\|x_i - x\|) - K_h(\|x_k - x\|)\right]\right| \\
        &\leq \left|\frac{\epsilon}{{N}^*_j + \epsilon}\max_{k}\left[K_h(\|x_i - x\|) - K_h(\|x_k - x\|)\right]\right|\\
        &\leq \left|\frac{\epsilon}{{N}^*_j + \epsilon}K_h(0)\right|.
\end{align}
\end{proof}

\begin{lemma}\label{lem:2}
Let $\hat{f}_j, f_j^*, N^{*}_{j}$ be defined as in Lemma~\ref{lem:1}. Given a fixed data sample of size $N$, for all $x$, $a > 0$ and $c > 0$ we have
\[
 \PP(|\hat{f}_j(x) - f_j^*(x)| \geq a) \leq 2 \exp\left\{-\frac{2(1+N^{*}_{j})^2a^2}{N K_h(0)^2}\right\}.
\]
\end{lemma}

\begin{proof}
Once the sample is fixed, $f^*_j(\cdot)$ becomes a deterministic function of the sample, and $N^*_j$ becomes a deterministic constant. Following \citealt{van-de-Geer-on-Hoeffding-for-dependent}, we define 
\begin{align}
    S_n &= \sum_{i=1}^{N} [\hat{f}_{ij}(x) - f_{ij}^*(x)], \\
    L_i &= 0,\\
    U_i &= \frac{1}{1+N^{*}_{j}}K_h(0),
\end{align}
where $\hat{f}_{ij}(x) - f_{ij}^*(x)$ denotes that the two functions only differ at the $i$th matrix entry, $L_i$ and $U_i$ are constant (degenerate) random variables for a fixed sample and
\begin{align}
    C_N^2 &= \sum_{i=1}^{N} (U_i - L_i)^2 \\
          &= \sum_{i=1}^{N} \left(\frac{1}{1+N^{*}_{j}}K_h(0)\right)^2 \\
          &= N \left(\frac{1}{1+N^{*}_{j}}K_h(0)\right)^2 \\
\end{align}
Then, for all $x$, $a > 0$ and $c > 0$, we have
\begin{align}
 \PP(|\hat{f}_j(x) - f_j^*(x)| \geq a , C_n^2 \leq c^2 \text{ for some $n$}) &\leq \PP(\hat{f}_j(x) - f_j^*(x) \geq a , C_n^2 \leq c^2 \text{ for some $n$}) +\notag{}\\
       \phantom{{} } &\phantom{{} \leq {}} \PP(-\hat{f}_j(x) + f_j^*(x) \geq a , C_n^2 \leq c^2 \text{ for some $n$}) \\
        &\leq \PP\left(\sum_{i=1}^n [\hat{f}_{ij}(x) - f_{ij}^*(x)] \geq a , C_n^2 \leq c^2 \text{ for some $n$}\right) +\notag{}\\
        \phantom{{} } &\phantom{{} \leq {}} \PP\left(\sum_{i=1}^n [-\hat{f}_{ij}(x) + f_{ij}^*(x)] \geq a , C_n^2 \leq c^2 \text{ for some $n$}\right) \\
        &= \PP(S_n \geq a , C_n^2 \leq c^2 \text{ for some $n$}) +\notag{}\\
        \phantom{{} } &\phantom{{} \leq {}} \PP(-S_n \geq a , C_n^2 \leq c^2 \text{ for some $n$}) \\
        &\leq 2 \exp\left\{-\frac{2a^2}{c^2}\right\}.
\end{align}

Given a fixed sample of size $N$, choose $c_0$ such that $C_n^2 \leq c_0^2$ for all $1 \leq n \leq N$. Then
\begin{align}
 \PP(|\hat{f}_j(x) - f_j^*(x)| \geq a , C_n^2 \leq c_0^2 \text{ for some $ 1 \leq n \leq N$}) \\
        &\leq \PP(|\hat{f}_j(x) - f_j^*(x)| \geq a , C_n^2 \leq c_0^2 \text{ for some $n$}) \\
        &\leq 2 \exp\left\{-\frac{2a^2}{c_0^2}\right\}.
\end{align}
Because $C_n^2 \leq c_0^2$ for all $1 \leq n \leq N$, we have 
\begin{align}
\PP(|\hat{f}_j(x) - f_j^*(x)| \geq a) &= \PP(|\hat{f}_j(x) - f_j^*(x)| \geq a , C_n^2 \leq c_0^2 \text{ for some $ 1 \leq n \leq N$}) \\
        &\leq 2 \exp\left\{-\frac{2a^2}{c_0^2}\right\}.
\end{align}
Having already establish that $C_N^2 \leq N \left(\frac{1}{1+N^{*}_{j}}K_h(0)\right)^2$, we set $c_0 = \sqrt{N \left(\frac{1}{1+N^{*}_{j}}K_h(0)\right)^2}$ and obtain
\begin{align}\label{res:1}
 \PP(|\hat{f}_j(x) - f_j^*(x)| \geq a) &\leq 2 \exp\left\{-\frac{2(1+N^{*}_{j})^2a^2}{N K_h(0)^2}\right\}.
\end{align}
\end{proof}

\begin{thm}
For a fixed data sample of size $N$, let $P^*(X|\ell^-)$ denote the optimal estimator based on that sample and $\widehat{P}(X|\ell^-)$ be the light cone estimator based on the same sample. Let $\widehat{P}(X|S_j)$ be bounded by a constant $M$ for all $X, j$. If $|\widehat{P}(S_j|\ell^-) - P^*(S_j|\ell^-)| \rightarrow 0$ for all $j$, then for any $X$, $\epsilon > 0$, $\delta > 0$, and sufficiently large $N$,
\[
    \PP\left(|\widehat{P}(X|\ell^-) - P^*(X|\ell^-)| > \epsilon\right) \leq 2\exp\left\{-\frac{2(1+N^{*}_{s})^2(\epsilon - \delta)^2}{N K_h(0)^2}\right\},
\]
where $N^*_{s}$ is the (smallest) sum of weights for the predictive states, and $K_h(\cdot)$ is a kernel of bandwidth $h$. 
\end{thm}

\begin{proof}
\begin{align}
    \left|\widehat{P}(X|\ell^-) - P^*(X|\ell^-)\right| &= \left|\sum_{j=1}^{K} \left[\widehat{P}(X|S_j)\widehat{P}(S_j|\ell^-) - P^*(X|S_j)P^*(S_j|\ell^-) \right]\right| \\
        &\leq \sum_{j=1}^{K} \left|\left[\widehat{P}(X|S_j)\widehat{P}(S_j|\ell^-) - P^*(X|S_j)P^*(S_j|\ell^-) \right]\right| \\
        &= \sum_{j=1}^{K} \left|\left[\left(\widehat{P}(X|S_j) - P^*(X|S_j)\right)P^*(S_j|\ell^-) + \left(\widehat{P}(S_j|\ell^-) - P^*(S_j|\ell^-)\right)\widehat{P}(X|S_j) \right]\right| \\
        &\leq \sum_{j=1}^{K} \left|\widehat{P}(X|S_j) - P^*(X|S_j)\right|P^*(S_j|\ell^-) + \sum_{K}^{j=1}\left|\widehat{P}(S_j|\ell^-) - P^*(S_j|\ell^-)\right|\widehat{P}(X|S_j) \\
        &\leq \max_{j}\left\{\left|\widehat{P}(X|S_j) - P^*(X|S_j)\right|\right\} + \max_{j}\left\{\left|\widehat{P}(S_j|\ell^-) - P^*(S_j|\ell^-)\right|\widehat{P}(X|S_j)\right\} \\
        &:= A + B. \\
\end{align}
Therefore,
\begin{align}
\PP\left(\left|\widehat{P}(X|\ell^-) - P^*(X|\ell^-)\right| > \epsilon\right) &\leq \PP\left(A + B > \epsilon\right) \\
    &= \PP\left(A > \epsilon - B\right) \\
    &= \PP\left(A > \epsilon - B | B \leq \delta \right)\PP(B \leq \delta) +  \PP\left(A > \epsilon - B | B > \delta \right)\PP(B > \delta)\\
    &\leq \PP\left(A > \epsilon - \delta\right)\PP(B \leq \delta) +  \PP\left(A > \epsilon - B, B > \delta\right).
\end{align}

For sufficiently large $N$, $\PP\left(A > \epsilon - B, B > \delta\right) = 0$ and $\PP(B \leq \delta) = 1$, given that $\widehat{P}(X|S_j)$ is bounded and $\left|\widehat{P}(S_j|\ell^-) - P^*(S_j|\ell^-)\right|\rightarrow0$. Therefore, given $N$ sufficiently large,  
\begin{align}
\PP\left(\left|\widehat{P}(X|\ell^-) - P^*(X|\ell^-)\right| > \epsilon\right) &\leq \PP\left(A > \epsilon - \delta\right) \\
        &= \PP\left(\max_{j}\left|\widehat{P}(X|S_j) - P^*(X|S_j)\right| > \epsilon - \delta\right) \\
        &\leq 2\exp\left\{-\frac{2(1+N^{*}_{j})^2(\epsilon - \delta)^2}{N K_h(0)^2}\right\} \\
        &\leq 2\exp\left\{-\frac{2(1+N^{*}_{s})^2(\epsilon - \delta)^2}{N K_h(0)^2}\right\},
\end{align}
where the penultimate inequality follows from Lemma~\ref{lem:2}.
\end{proof}

\newpage
\twocolumn
\section*{Appendix II: Implementation Details}

We now discuss the choosing of various parameter settings for the two algorithms, as well as some computational techniques used to improve runtime performance.

\subsection*{Choosing Number of States}

In both mixed LICORS and Moonshine a user must specify the maximum number of predictive states for the model, which effectively controls the complexity of the model. In OHP, one must specify the exact number of predictive states, since the number is determined by a $k$-means++~\citep{arthur2007k} clustering step. In all cases, this number can be chosen based on user preference for simpler models, or cross-validation may be used to find the number of states that gives the best predictive performance on held-out data.

\subsection*{Dimensionality Reduction Choice in Moonshine}

Another parameter that must be chosen is the degree of dimensionality reduction when forming distribution signatures in Moonshine. Data can guide this choice (through cross-validation), or user preference for more compact models can guide the choice for greater degrees of dimensionality reduction. The fewer the number of dimensions, the less discriminative the signatures, and thus, the higher the likelihood of merging clusters.

\subsection*{Density Based Clustering Considerations}

When using density based clustering such as DBSCAN~\citep{DBSCAN}, two issues arise. First, a suitable local neighborhood size must be chosen (controlled by an $\epsilon$ parameter). Second, such methods can be computationally expensive and thus slow. To address the first issue, we take an iterative search approach by beginning with very small neighborhood sizes, then increase them until a significant portion of the data is clustered, but keep the proportion below 100\%. To address the second issue, we use DBSCAN to cluster only a seed portion of all observations, then assign remaining observations to nearest cluster centers, which greatly improves runtime. Controlling the proportion of data used for seeding versus the portion assigned to cluster centers affects the degree of forced convexity of resulting clusters, and also determines the total runtime of the clustering. Fewer seed points results in faster clustering, but with more convex-shaped (e.g., $k$-means-like) clusters.

\subsection*{Scaling}

Since Moonshine and OHP cluster based on distances, it becomes important to normalize the scaling of all axes and dynamic ranges of all experiments. Additionally, if the scale of training light cones differs from the scale of test light cones predictive performance will suffer.

\subsection*{Nonparametric Density Estimation}

Nonparametric density estimation techniques are instance based and slow with increasing numbers of observations. Our algorithms use kernel density estimators~\citep{KDE1,KDE2}, for which we only retain a randomly chosen subsample of five hundred points in each cluster to compute the densities. The resulting systems still perform well, as shown in \S \ref{sec:results}, while being computationally tractable.

\end{document}